\documentclass[final]{llncs}
\usepackage[T1]{fontenc}
\usepackage{amsmath,amssymb}
\usepackage{graphicx}
\usepackage{bbding}

\usepackage{tikz}
\usepackage{apxproof}
\usepackage[utf8]{inputenc} 
\usepackage[T1]{fontenc}    
\usepackage{hyperref}       
\usepackage{url}            
\usepackage{booktabs}       
\usepackage{amsfonts}       
\usepackage{nicefrac}       
\usepackage{microtype}      
\usepackage{lipsum}
\usepackage{fancyhdr}       
\usepackage{subfig,comment}
\graphicspath{{media/}}     
\usepackage[capitalize]{cleveref}

\newcommand{\argmax}{\mathop{\rm arg~max}\limits}

\newtheorem{assumption}{Assumption}
\newtheorem{Theorem}{Theorem}

\usepackage{xparse}

\NewDocumentEnvironment{theoremproof}{O{}}{}{}
\NewDocumentEnvironment{theoremproofref}{O{}}{}{}
\newif\ifarxivversion

\arxivversiontrue 

\ifarxivversion 
\excludecomment{theoremproofref}
\else
\excludecomment{theoremproof}
\fi

\begin{document}
\title{Robustness bounds on the successful adversarial examples in probabilistic models: Implications from Gaussian processes}
%

\author{Hiroaki Maeshima\inst{1,2} \Envelope \and Akira Otsuka\inst{1}\orcidID{0000-0001-6862-2576}}
\authorrunning{H. Maeshima \and A. Otsuka}
%
\institute{Institute of Information Security, Yokohama, Japan \\
\email{mgs214505@iisec.ac.jp, otsuka@ai.iisec.ac.jp}\and
NTT TechnoCross Corporation, Yokohama, Japan }
\maketitle              
\begin{abstract}
   Adversarial example (AE) is an attack method for machine learning, which is crafted by adding imperceptible perturbation to the data inducing misclassification. In the current paper, we investigated the upper bound of the probability of successful AEs based on the Gaussian Process (GP) classification, a probabilistic inference model. We proved a new upper bound of the probability of a successful AE attack that depends on AE's perturbation norm, the kernel function used in GP, and the distance of the closest pair with different labels in the training dataset. Surprisingly, the upper bound is determined regardless of the distribution of the sample dataset. 
   We showed that our theoretical result was confirmed through the experiment using ImageNet. In addition, we showed that changing the parameters of the kernel function induces a change of the upper bound of the probability of successful AEs.

\keywords{Gaussian Processes \and Adversarial Examples \and Robustness Bounds}
\end{abstract}
\section{Introduction}
\subsection{Adversarial Example}
Today, machine learning is widely used in various fields, and concerns about its security have emerged.
An adversarial example (AE) is widely investigated among such attack methods \cite{goodfellow_explaining_2015}. AE is a sample that is slightly different from a natural sample that is misclassified by a machine learning classifier \cite{goodfellow_explaining_2015}. AE is often crafted against neural networks by adding a small perturbation (adversarial perturbation) to the original input, and various methods are proposed to make an adversarial perturbation \cite{carlini_towards_2017,goodfellow_explaining_2015}. 

Since AE is considered an attack method against machine learning, various defense methods against AE have been investigated, including the detection method \cite{feinman_detecting_2017} and adversarial training \cite{cai_curriculum_2018}.
However, Carlini \& Wagner \cite{carlini_adversarial_2017} reviewed ten detection methods and stated that no defense method could survive a white-box AE attack, an attack using the knowledge of the victim's architecture. Therefore, both the defense methods and their theoretical basis should be required for the defense against AE.

In the current study, we investigated the theoretical basis of AE using the Gaussian Process (GP), a stochastic process whose output is multivariate normally distributed. GP can be used for a probabilistic model for classification and regression \cite{rasmussen_gaussian_2006}, and previous research shows that GP is equivalent to some types of classifier such as linear regression and a relevance vector machine, and especially, GP is equivalent to a certain type of neural networks \cite{williams_computing_1996}. More specifically, the activation function in neural networks corresponds to the kernel function in the NN \cite{williams_computing_1996}.

\subsection{Related Researches}
Until now, some studies have been conducted that aim to establish the theoretical basis of AE. Shafahi, Huang, Studer, Feizi \& Goldstein \cite{shafahi_are_2019} have suggested that AE is inevitable under a certain condition on the distribution of the training data, using the geometrical method. Blaas et al. \cite{blaas_adversarial_2020} have shown the algorithmic method to calculate the upper and lower bound of the value of GP classification in an arbitrary set. Cardelli, Kwiatkowska, Laurenti \& Patane \cite{cardelli_robustness_2019} have shown the upper bound of the probability of robustness of the GP output in an arbitrary set. Wang, Jha \& Chandhuri \cite{wang_analyzing_2018} have shown the method of creating a robust dataset when using a k-nearest neighborhood classifier.
Fawzi, Fawzi \& Fawzi \cite{fawzi_adversarial_2018} have shown the fundamental upper bound on the robustness of the classification, which some studies are based on. Mahloujifar, Diochnos \& Mahmoody \cite{mahloujifar_curse_2019} have applied their results to broader types of data distributions, and Zhang, Chen, Gu \& Evans \cite{zhang_understanding_2020} have shown the intrinsic robustness bounds for classifiers with a conditional generative model.
More specifically, Gilmer et al. \cite{gilmer_relationship_2018} have shown a fundamental bound relating to the classification error rate in the field of neural networks; Bhagoji, Cullina \& Mittal \cite{bhagoji_lower_2019} have shown lower bounds of adversarial classification loss using optimal transport theory.

Some studies suggest certified defense methods. Randomized smoothing \cite{cohen_certified_2019} is a method for composing "smoothed" classifiers that are derived from arbitrary classifiers. The prediction of an arbitrary input by smoothed classifiers has a safety radius, within which the classification results of the neighborhood inputs are the same as the central input.

Our approach is different from the related studies in the following viewpoints.
First, our approach focuses on the training data for the classifier and is easier to use practically.
For example, the other studies focus on the geometry of the classifier \cite{shafahi_are_2019} or the upper bound for which some cumbersome calculation with all training data is necessary \cite{blaas_adversarial_2020,cardelli_robustness_2019}.
Using our approach, we can analyze how the robustness will change when the distribution of training data of the classifier changes, for example, when some training data are removed from the dataset or moved within the input space.
Cohen, Rosenfeld \& Kolter's randomized classifier \cite{cohen_certified_2019} needs Monte Carlo samplings to calculate the variance used in the calculation of the safety radius. However, our approach can provide the predictive variance using GP. Thus, our approach can be used with Cohen, Rosenfeld \& Kolter's randomized smoothing.

Second, our approach can be applied to a wide range of classifiers thanks to GP.
For example, Wang, Jha \& Chandhuri 's approach  \cite{wang_analyzing_2018} also focuses on how robustness changes when the distribution of training data of the classifier changes, but their approach is based on the nearest-neighbor method. Our approach is based on GP, which can be applied to broader inference models.
Smith, Grosse, Backes \& Alvarez \cite{smith_adversarial_2023} have shown adversarial vulnerability bounds for binary GP classification. However, Smith, Grosse, Backes \& Alvarez's proof is based on the Gaussian kernel, while the result of our research is based on the characteristics of the kernel functions, which includes the Gaussian kernel as a special case. Therefore, the current study has a broader application than \cite{smith_adversarial_2023}.

\subsection{Contributions}
Our research has the following contributions.
\begin{itemize}
    \item Our research shows that the success probability of an AE attack against a certain dataset in the GP classification with a certain kernel function is upper-bounded by the function of the distance of the nearest points that have different labels. 
    \item We confirmed our theoretical result through the experiment using ImageNet with various kernel parameters in GP classification. The experimental result is well-suited to the theoretical result.
    \item We showed that changing the parameters of the kernel function causes the change of the theoretical upper bound, and thus our result gives the theoretical basis for the enhancement method of robustness.

\end{itemize}

\section{Theoretical Results}

\subsection{Problem Formulation and Assumption}

Let $\mathcal{D}$ be a dataset with $N$ samples that has data points $x_i$ and labels of object classes $y_i$. The data point $x_i$ is a $D$-dimensional vector, and the label $y_i$ is binary. Namely,
\begin{equation}
\mathcal{D}=\{\{x_1,y_1\},\{x_2,y_2\},...,\{x_N,y_N\}\},
x_i \in \mathbb{R}^D, y_i \in \{+1, -1\}. 
\end{equation}

Let
\begin{equation}
\begin{split}
&\mathcal{D_+}=\{\{x_i,y_i\} \in \mathcal{D}\mid y_i=+1\}, \\
&\mathcal{D_-}=\{\{x_i,y_i\} \in \mathcal{D}\mid y_i=-1\}  \\
\end{split}
\end{equation}
for further simplicity.

We consider the binary classification of the dataset, using GP regression with a certain kernel function $k(x,x')$ \cite{rasmussen_gaussian_2006}. We define a GP regressor $\mathcal{R}(x): \mathbb{R}^D \rightarrow \mathcal{N}(y;\mu_\mathcal{R},\sigma^2_\mathcal{R})$. $\mathcal{R}(x)$ gives the output as a Gaussian distribution $\mathcal{N}(y;\mu_\mathcal{R},\sigma^2_\mathcal{R})$.
Then we define a probabilistic GP classifier $\mathcal{C}(\cdot): \mathbb{R}^D \rightarrow \{+1,-1\}$, which is constructed as below:
\begin{equation}
\begin{split}
    \mathcal{C}(x) :=
    \begin{cases}
+1 &\mathrm{if} ~ p \sim  \mathcal{N}(x) \geq 0\\
-1 &\mathrm{otherwise}
    \end{cases}
\end{split}
\end{equation}
where $p$ is a value sampled from the distribution $\mathcal{N}(x)$. 
Intuitively, $\mathcal{C}(x)$ is a probabilistic binary classifier whose output is $+1$ when a sample from the output distribution of $\mathcal{R}(x)$ is not negative.

Let the maximum value of the kernel function between two input points of $\mathcal{D}$ with different labels be $s$.
We write the definition of $s$ as follows, without loss of generality:

The value of the kernel function between 2 points \begin{math}x_+ \in \mathcal{D_+}, x_- \in \mathcal{D_-} \end{math} is described as \textit{s}, where \begin{math}x_- \in \mathcal{D}_-\end{math} gives the maximum value of $k(x_+,x_-)$ when $x_+$ is fixed.
$s$ can be written as a function of $x_+$, that is
$$s = \max_{x_- \in \mathcal{D}_-}k(x_+,x_-).$$


Put $x_* \in \mathbb{R}^D$ such that $k(x_+,x_*) = r$, where $r$ is a constant.

In this paper, we introduce the following assumption.
\begin{assumption}[$\epsilon$-proximity]
     Let $\mu_\mathcal{R}$ be the predictive mean of $x_*$ by GP regression trained with $\mathcal{D}$, and $\mu_{\mathcal{R}2}$ be the corresponding value by GP regression trained with $\mathrm{\{\{}x_+,+1\mathrm{\},\{}x_-,-1\mathrm{\}\}}$. Then, for all $x_* \in \mathbb{R}^D$ such that $k(x_+,x_*) = r$,
    \begin{equation}
        |\mu_\mathcal{R} - \mu_{\mathcal{R}2}| \leq \epsilon ~\mathrm{and}~ \mu_\mathcal{R}>0
    \end{equation}
    holds with some $\epsilon \geq 0$.
    \label{assumption:1}
\end{assumption}

Intuitively, the former part of \cref{assumption:1} suggests that the predictive mean of the GP regressor with all training data is close to that with only two training data, implying that the effect of training data other than $x_+$, $x_-$ is small.

\subsection{Maximum Success Probability of AE}
\begin{Theorem}
\label{thm:msp}
Consider $\mathcal{R}(x)$ with kernel function $k(x,x')$ trained with $\mathcal{D}$.
Let $x_+$ be arbitrarily chosen from $\mathcal{D_+}$ and $x_-$ be the data point of $\mathcal{D_-}$, such that $k(x_+,x_-)$ is the greatest value when $x_+$ is fixed and $x_-$ is taken from  $\mathcal{D_-}$. Then, for any $x_* \in \mathbb{R}^D$ such that $k(x_+,x_*) = r$ and $k(x_+,x_*) > k(x_-,x_*) + \epsilon(k(x_+,x_+) - k(x_+,x_-))$, the upper bound of the probability that $\mathcal{C}(x_*)=-1$ is upper-bounded with Maximum Success Probability (MSP) function $\phi$ as

\begin{equation}
\Pr(\mathcal{C}(x_*) = -1) \,< \, \Phi\left(-\frac{\mu}{\sigma}\right) \,<\, \frac{1}{2}\exp{\left(-\frac{\mu^2}{2\sigma^2}\right)} \,=\, \phi(r|\mathcal{D}),
\label{eq:msp}
\end{equation}

where the notations below are used.
\begin{equation}
\Phi(x) = \frac{1}{\sqrt{2\pi}}\int_{-\infty}^x\exp{\left(-\frac{u^2}{2}\right)}du,
\end{equation}

\begin{equation}
\mu = \frac{k(x_+,x_{*max})-k(x_-,x_{*max})}{k(x_+,x_+)-k(x_+,x_-)} - \epsilon,
\end{equation}

\begin{equation}
\begin{split}
\sigma^2 = &k(x_+,x_+) - 
\frac{k(x_+,x_+)({k(x_+,x_{*max})}^2+{k(x_-,x_{*max})}^2)}{{k(x_+,x_+)}^2 - {{k(x_+,x_-)}^2}}+\\
&\frac{2k(x_+,x_-)k(x_+,x_{*max})k(x_-,x_{*max})}{{k(x_+,x_+)}^2 - {{k(x_+,x_-)}^2}},\\
\end{split}
\end{equation}

\begin{equation}
x_{*max} = \argmax_{x_*}[k(x_-,x_*)].
\end{equation}

The above bound holds for any translation-invariant kernel function $k(x,x')$, that is, a kernel function that satisfies $k(x,x)$ is constant for all $x \in \mathbb{R}^D$.
\end{Theorem}

The schematic diagram is shown in \cref{fig:thrmsp}.

\begin{figure}
    \centering
\begin{tikzpicture}[scale=0.5]
    \draw[thick, ->] (-2, 0)--(10, 0) node[right]{};
    \draw[thick, ->] (0,-0.5)--(0,7) node[above] {};
    \draw[dashed](2.6,7)--(4.6,-.5);
    \coordinate [label=above left:{O}] (origin) at (0,0);
    \coordinate [label={[blue]\large $\mathcal{D}_+$}] (bluelb) at (1.8,6.0);
    \coordinate [label={[orange]\large $\mathcal{D}_-$}] (orangelb) at (4.1,6.0);
    \coordinate [label=below left:{$x_+$}] (x1) at (3,0.8);
    \coordinate [label=below right:{$x_-$}] (x1) at (5,2);
    \coordinate [label=below:{$x_*$}] (x1) at (3.9,1.2);

    \draw[very thick] (3,0.8)--(5,2);

    \draw ([shift={(5.2,-0.6)}]94:2.607) arc[radius=2.607, start angle=94, end angle= 150];
    \draw ([shift={(2.8,2.04)}]280:1.26) arc[radius=1.26, start angle=280, end angle= 320];
    
    \coordinate [label=above:{$s$}] (rlabel) at (3.7,1.6);
    \coordinate [label=below:{$r$}] (slabel) at (3.4,0.9);
    
    \fill[blue] (-0.5,3) circle (0.12);
    \fill[blue] (1,2.0) circle (0.12);
    \fill[blue] (2,3.6) circle (0.12);
    \fill[blue] (3,0.8) circle (0.12);
    \fill[blue] (3.4,2.8) circle (0.12);
    
    \fill[orange] (5,2) circle (0.12);
    \fill[orange] (8.5,3) circle (0.12);
    \fill[orange] (5.4,3.7) circle (0.12);
    \fill[orange] (7,1) circle (0.12);
    \fill[orange] (4,3.8) circle (0.12);
    
    \fill[green] (3.8,1.28) circle (0.12);
\end{tikzpicture}
    \caption{\small{Illustration of the conditions in Theorem \ref{thm:msp} assuming the input space as $\mathbb{R}^2$. $x_+$ and $x_-$ are input data points from $\mathcal{D}_+$ and $\mathcal{D}_-$ respectively. Note that $r=k(x_+,x_*)$ and $s=k(x_+,x_-)$.}}
    \label{fig:thrmsp}
\end{figure}
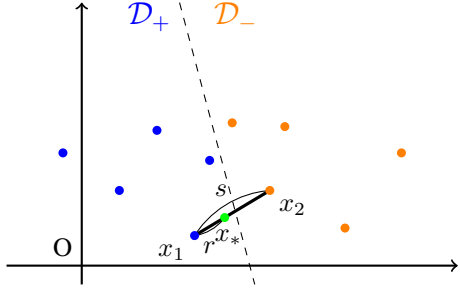

\textit{Proof sketch.} First, we prove that the prediction variance increases if the training point increases in GP regression (Lemma 1). In the next, we prove that $\operatorname{Erf}(0;\mu,\sigma) = \int^0_{-\infty}\mathcal{N}(x;\mu, \sigma^2)dx$ increases monotonically with respect to $\sigma^2$ if $\mu>0$ (Lemma 2). Then we prove that the probability that $x_*$ is classified as $-1$ increases monotonically with respect to $k(x_*,x_-)$ (Lemma 3), and the theorem is proved. \qed

\begin{theoremproofref}
The complete proof will be found in the full version of this paper \cite{maeshima2024robustnessboundssuccessfuladversarial}.
\end{theoremproofref}

\begin{theoremproof}
    \begin{toappendix}
    \subsection{Proof of Theorem 1}
    \label{sec:proof1}
    In the proof, we set the notations below.
    
    We construct a GP regressor $\mathcal{R}(x)$ with the kernel function $k(x,x')$. In GP regressor, the mean $\mu_\mathcal{R}$ and the variance $\sigma_\mathcal{R}^2$ of the new data point $x_*$ can be calculated as
    
    \begin{equation}
    \mu_\mathcal{R} = k(x_*)^TK^{-1}\textbf{y}
    \label{eq0}
    \end{equation}
    \begin{equation}
    \sigma_\mathcal{R}^2 = k(x_*,x_*) - k(x_*)^TK^{-1}k(x_*)
    \label{eq1}
    \end{equation}
    
    where 
    \begin{equation}
    k(x_*) = \left(k(x_+,x_*),\cdots,k(x_n,x_*)\right)^T,
    \end{equation}
    \begin{equation}
    K = \begin{pmatrix}
    k(x_+,x_+) & \cdots & k(x_+,x_n) \\
    \cdots & \cdots & \cdots \\
    k(x_n,x_+) & \cdots & k(x_n,x_n) \\
    \end{pmatrix},\\
    \end{equation}
    
    \begin{equation}
    \textbf{y}=\left(y_1,\cdots,y_n\right).
    \end{equation}
    
    In the next we construct GP classifier $\mathcal{C}(x)$ whose prediction is $y_*=+1$ if a sample from $\mathcal{R}(x)$ is greater than 0, and the prediction is $y_*=-1$ otherwise.
    
    For further calculation, set the below notations
    \begin{equation}
    \begin{split}
    &\theta_1 = k(x_*,x_*),~\theta_{r1} = k(x_+,x_*),\\
    &\theta_{r2} = k(x_-,x_*),~\theta_s = k(x_+,x_-).\\
    \end{split}
    \end{equation}
    
    The following lemma is a simpler version of the statement Exercise 4 of Chapter 2 from \cite{rasmussen_gaussian_2006}. The proof is in the supplemental material.
    \begin{lemma}
    Let $\mathrm{Var}_n(x_*)$ be the predictive variance of a GP regression at $x_*$ given a training dataset of size $n$, and $\mathrm{Var}_{n-1}(x_*)$ be the correspondent variance given only the first $n-1$ points of the training dataset. Then, $\mathrm{Var}_n(x_*) < \mathrm{Var}_{n-1}(x_*)$ holds.
    \label{lemma:varmin}
    
    \end{lemma}
    
    \begin{proof}
    Let
    \begin{equation}
    K_n = \begin{pmatrix}
    k(x_+,x_+) & \cdots & k(x_+,x_n) \\
    \vdots & \ddots & \vdots \\
    k(x_+,x_n) & \cdots & k(x_n,x_n) \\
    \end{pmatrix},
    \end{equation}
    \begin{equation}
    k_{*n}={(k(x_+,x_*),\cdots,k(x_n,x_*))}^T
    \end{equation}
    
    \begin{equation}
    \operatorname{Var}(x_{*n-1}) = k(x_*,x_*) - {k_{*n-1}}^T K_{n-1}^{-1} k_{*n-1},
    \end{equation}
    
    \begin{equation}
    \operatorname{Var}(x_{*n}) = k(x_*,x_*) - {k_{*n}}^T K_n^{-1} k_{*n}
    \end{equation}

    $K_n$ can be written as 
    \begin{equation}
    K_n = \begin{pmatrix}
    K_{n-1} & B \\
    B^T & k(x_n,x_n)
    \end{pmatrix}, 
    \end{equation}
    where $B={(k(x_+,x_n),\cdots,k(x_{n-1},x_n))}^T$.
    
    Therefore, $K_n^{-1}$ can be decomposed as
    \begin{equation}
    K_n^{-1} = \begin{pmatrix}
    \Tilde{A} & \Tilde{B} \\
    \Tilde{B}^T & \Tilde{D}
    \end{pmatrix}, 
    \end{equation}
    where
    \begin{equation}
    \left\{
    \begin{array}{l}
        \Tilde{A} = K_{n-1}^{-1} + K_{n-1}^{-1}BMB^TK_{n-1}^{-1}\\
        \Tilde{B} = -K_{n-1}^{-1}BM\\
        \Tilde{D} = M\\
        M = (k(x_n,x_n)-B^TK_{n-1}^{-1}B)^{-1}
    \end{array}
    \right.
    \end{equation}
    
    Now
    \begin{equation}
    \begin{split}
    &{k_{*n}}^T K_n^{-1} k_{*n} = 
    {k_{*n-1}}^T \Tilde{A}{k_{*n-1}}
    + 2k(x_n,x_*)\Tilde{B}^T{k_{*n-1}} + k(x_n,x_*)^2\Tilde{D}.
    \end{split}
    \end{equation}
    
    Decomposing A results in
    \begin{equation}
    \begin{split}
    \operatorname{Var}(x_{*n}) = &k(x_*,x_*) - ( {k_{*n-1}}^T K_{n-1}^{-1} {k_{*n-1}} + \\
    &{k_{*n-1}}^TK_{n-1}^{-1}BMB^TK_{n-1}^{-1}{k_{*n-1}}+\\
    &2k(x_n,x_*)\Tilde{B}^T{k_{*n-1}} + k(x_n,x_*)^2\Tilde{D})\\
    \operatorname{Var}(x_{*n}) =  &\operatorname{Var}(x_{*n-1}) - ({k_{*n-1}}^TK_{n-1}^{-1}BMB^TK_{n-1}^{-1}{k_{*n-1}}+\\
    &2k(x_n,x_*)\Tilde{B}^T{k_{*n-1}} + k(x_n,x_*)^2\Tilde{D}).\\
    \end{split}
    \label{eq:varxn}
    \end{equation}
    
    Now, proving the Lemma is equivalent to prove
    \begin{equation}
    \begin{split}
    &{k_{*n-1}}^TK_{n-1}^{-1}BMB^TK_{n-1}^{-1}{k_{*n-1}}+ 2k(x_n,x_*)\Tilde{B}^T{k_{*n-1}} + k(x_n,x_*)^2\Tilde{D}>0.
    \end{split}
    \label{eqtoproof}
    \end{equation}
    
    We now prove Equation \ref{eqtoproof}. It can be decomposed as:
    \begin{equation}
    \begin{split}
    &M\left({k_{*n-1}}^TK_{n-1}^{-1}BB^TK_{n-1}^{-1}{k_{*n-1}}-2k(x_n,x_*)B^TK_{n-1}^{-1}{k_{*n-1}} + k(x_n,x_*)^2\right)
    \end{split}
    \label{eq:toproof2}
    \end{equation}
    ($\because$ M is $1 \times 1$ matrix, therefore M can be regarded as a scalar.)
    
    Now $B^TK_{n-1}^{-1}{k_{*n-1}}$ is a scalar. Then, consider $(B^TK_{n-1}^{-1}{k_{*n-1}}-k(x_n,x_*))^2$.
    
    \begin{equation}
    \begin{split}
    (B^T&K_{n-1}^{-1}{k_{*n-1}}-k(x_n,x_*))^2 =\\
    &{k_{*n-1}}^TK_{n-1}^{-1}BB^TK_{n-1}^{-1}{k_{*n-1}}-2k(x_n,x_*)B^TK_{n-1}^{-1}{k_{*n-1}} + k(x_n,x_*)^2
    \end{split}
    \end{equation}
    
    Now consider
    $M = (k(x_n,x_n)-B^TK_{n-1}^{-1}B)^{-1}$. This means the predictive variance by GP regression of $x_n$ with a dataset which includes $x_+,x_-,\cdots,x_{n-1}$. Therefore, $M>0$.
    
    Plugging $M>0$ and $(B^TK_{n-1}^{-1}{k_{*n-1}}-k(x_n,x_*))^2>0$ into Equation \ref{eq:toproof2},
    \begin{equation}
    \begin{split}
    &({k_{*n-1}}^TK_{n-1}^{-1}BMB^TK_{n-1}^{-1}{k_{*n-1}}+ 2k(x_n,x_*)\Tilde{B}^T{k_{*n-1}} + k(x_n,x_*)^2\Tilde{D})>0
    \end{split}
    \end{equation}
    holds. Hence the lemma was proved.
    %
    %
    
    \end{proof}
    
    Thus, if the two points ($x_+, x_-$) in the training dataset of GP regression are fixed, the maximum predictive variance is the variance calculated with the two points as the training dataset only.
    
    Next, we calculate the predicted mean under the condition that only $x_+$ and $x_-$ are used as the training dataset.
    We write the predictive mean as $\mu_{\mathcal{R}2}$.
    
    Then $k(x_*)$ and \textbf{y} are written as $k(x_*) = (\theta_{r1},\theta_{r2})^T$, $\textbf{y}=(+1,-1)$. Using these, the predicted mean is calculated below, plugging them into \cref{eq0}.
    
    \begin{equation}
    \mu_{\mathcal{R}2} = k(x_*)^TK^{-1}y = \frac{\theta_{r1}-\theta_{r2}}{\theta_{1}-\theta_{s}}
    \label{mean}
    \end{equation}
    
    The category boundary is regarded as the point where the predicted mean is 0, therefore on the category boundary
    $\frac{\theta_{r1}-\theta_{r2}}{\theta_{1}-\theta_{s}} = 0$ holds.
    
    Therefore, category boundary is located where $\theta_{r1}=\theta_{r2}.$
    
    This implies that the set of the points which satisfies $k(x_+,x_*) =  k(x_-,x_*)$ is the category boundary, and \cref{mean} implies that where $k(x_+,x_*) >  k(x_-,x_*)$ the mean is positive.
    
    Similarly, the prediction variance with only $x_+,x_-$ as the training dataset is described as follows, plugging the values into \cref{eq1}:
    
    \begin{equation}
    \sigma_{\mathcal{R}2}^2  =\theta_1 - \frac{1}{{\theta_1}^2 - {{\theta_s}^2}}(\theta_1({\theta_{r1}}^2+{\theta_{r2}}^2)-2\theta_s\theta_{r1}\theta_{r2})
    \label{eq:var2}
    \end{equation}
    
    Considering \cref{lemma:varmin}, \cref{eq:var2} gives the maximum predictive variance of GP regression with the training dataset including $x_+,x_-$.
    
    The probability of $\mathcal{C}(x_*)=-1$ is the probability that the value of a sample from the output distribution is lower than 0.
    Therefore, it can be calculated as follows, using the cumulative distribution function of Gaussian distribution
    \begin{equation}
    \operatorname{Erf}(0;\mu,\sigma) = \int^0_{-\infty}\mathcal{N}(x;\mu, \sigma^2)dx
    \end{equation}
    
    where
    \begin{equation}
    \mathcal{N}(x;\mu, \sigma^2) = \frac{1}{\sqrt{2\pi\sigma^2}}
    \exp{\left(-\frac{(x-\mu)^2}{2\sigma^2}\right)}.
    \end{equation}
    
    \begin{lemma}
    $\operatorname{Erf}(0;\mu,\sigma)$ is monotonically increasing with respect to $\sigma^2$, under the condition $\mu > 0$.
    \label{lemma:varub}
    \end{lemma} 
    \begin{proof}
    
    Let 
    \begin{equation}
    \operatorname{Erf}(0;\mu,\sigma) = \int^0_{-\infty}\mathcal{N}(x;\mu, \sigma^2)dx
    \end{equation}
    and put
    \begin{equation}
    g(\sigma) = \frac{\partial}{\partial\sigma^2} \int^0_{-\infty}\mathcal{N}(x;\mu, \sigma^2)dx.
    \label{gsigma}
    \end{equation}
    
    \Cref{gsigma} can be rewritten as:
    
    \begin{equation}
    g(\sigma) = \int^0_{-\infty}\frac{\partial}{\partial\sigma^2} \frac{1}{\sqrt{2\pi\sigma^2}}
    \exp{\left(-\frac{(x-\mu)^2}{2\sigma^2}\right)}dx
    \end{equation}
    and calculating partial differentiation,
    \begin{equation}
    \begin{split}
    &g(\sigma) =\\
    &\frac{1}{2}(2\pi)^{-\frac{1}{2}}\sigma^{-3}\int^0_{-\infty}\exp{\left(-\frac{(x-\mu)^2}{2\sigma^2}\right)}\left( \frac{(x-\mu)^2}{\sigma^2}-1 \right)dx.
    \end{split}
    \label{eqgauss}
    \end{equation}
    
    Then we show that $g(\sigma) > 0$ under the condition $\mu > 0$.

    \paragraph{calculating integral when $\mu=0$}
    Let $\mu=0$.
    
    Then $g(\sigma)$ can be written as
    
    \begin{equation}
    g(\sigma) = \frac{1}{2}(2\pi)^{-\frac{1}{2}}\sigma^{-3}\int^0_{-\infty}\exp{\left(-\frac{x^2}{2\sigma^2}\right)}\left( \frac{x^2}{\sigma^2}-1 \right)dx
    \end{equation}
    
    Calculating the integral above, using these integral(Gaussian Integral)
    \begin{equation}
    \begin{split}
    &\int_{-\infty}^0\exp\left({-ax^2}\right) = \frac{1}{2}\sqrt{\frac{\pi}{a}}, \\
    &\int_{-\infty}^0x^2\exp\left({-ax^2}\right) = \frac{1}{4a}\sqrt{\frac{\pi}{a}}.
    \end{split}
    \end{equation}
    
    Now
    \begin{equation}
    \int^0_{-\infty}\exp{\left(-\frac{(x-\mu)^2}{2\sigma^2}\right)}\left( \frac{(x-\mu)^2}{\sigma^2}-1 \right)dx
    \end{equation}
    can be written as below, under condition $\mu=0$.
    \begin{equation}
    \frac{1}{\sigma^2}\int^0_{-\infty}x^2\exp{\left(-\frac{x^2}{2\sigma^2}\right)}-\int^0_{-\infty}\exp{\left(-\frac{x^2}{2\sigma^2}\right)}dx
    \end{equation}
    
    Then the Gaussian integral is applied, and the integral is written as below:
    
    \begin{equation}
    \begin{split}
    &\frac{1}{\sigma^2}\int^0_{-\infty}x^2\exp{\left(-\frac{x^2}{2\sigma^2}\right)}-\int^0_{-\infty}\exp{\left(-\frac{x^2}{2\sigma^2}\right)}dx \\ &= \frac{1}{\sigma^2}\cdot\frac{1}{4}\cdot2\sigma^2\sqrt{\pi}\sqrt{2\sigma^2} - \frac{1}{2}\sqrt{2\sigma^2}\sqrt{\pi} \\ &= 0
    \end{split}
    \label{eqzero}
    \end{equation}
    
    \paragraph{calculating integral when $\mu>0$}
    Let $\mu>0$ and $\sigma>0$.
    
    \Cref{eqgauss} can be written as below, using integration by substitution $y = x - \mu$.
    
    \begin{equation}
    \int^{-\mu}_{-\infty}\exp{\left(-\frac{y^2}{2\sigma^2}\right)}\left( \frac{y^2}{\sigma^2}-1 \right)\frac{dx}{dy}dy
    \end{equation}
    
    $\frac{dx}{dy} = 1$, and this is calculated as
    \begin{equation}
    \begin{split}
    &\int^{-\mu}_{-\infty}\exp{\left(-\frac{y^2}{2\sigma^2}\right)}\left( \frac{y^2}{\sigma^2}-1 \right)dy \\
    &=\int^{-\mu}_{-\infty}\frac{y^2}{\sigma^2}\exp{\left(-\frac{y^2}{2\sigma^2}\right)}dy - \int^{-\mu}_{-\infty}\exp{\left(-\frac{y^2}{2\sigma^2}\right)}dy\\
    &=\int^{0}_{-\infty}\frac{y^2}{\sigma^2}\exp{\left(-\frac{y^2}{2\sigma^2}\right)}dy - \int^0_{-\infty}\exp{\left(-\frac{y^2}{2\sigma^2}\right)}dy \\
    &- \left( \int^{0}_{-\mu}\frac{y^2}{\sigma^2}\exp{\left(-\frac{y^2}{2\sigma^2}\right)}dy - \int^0_{-\mu}\exp{\left(-\frac{y^2}{2\sigma^2}\right)}dy \right)\\
    &=\int^0_{-\mu}\exp{\left(-\frac{y^2}{2\sigma^2}\right)}dy - \int^{0}_{-\mu}\frac{y^2}{\sigma^2}\exp{\left(-\frac{y^2}{2\sigma^2}\right)}dy
    \end{split}
    \label{eqA1}
    \end{equation}
    
    This is calculated by two conditions: a) $\mu < \sigma$ and b) $\mu > \sigma$.
    
    \paragraph{a) $\mu < \sigma$}
    
    \Cref{eqA1} can be written as:
    \begin{equation}
    \int^0_{-\mu}\left(1-\frac{y^2}{\sigma^2}\right)\exp{\left(-\frac{y^2}{2\sigma^2}\right)}dy.
    \label{eqA2}
    \end{equation}
    In the interval $-\mu < y < 0$, $\left(1-\frac{y^2}{\sigma^2}\right) > 0$ therefore the integrand $>0$ all over the integral interval. Therefore the value of \cref{eqA2} is bigger than 0 regardless of $\sigma$.
    
    \paragraph{b) $\mu > \sigma$}
    From the \cref{eqzero}
    \begin{equation}
        \int_{-\infty}^0\exp{\left(-\frac{y^2}{2\sigma^2}\right)}\left(1-\frac{y^2}{\sigma^2}\right)dy = 0,
    \end{equation}
    
    Therefore
    \begin{equation}
    \begin{split}
        \int_{-\infty}^{-\mu}&\exp{\left(-\frac{y^2}{2\sigma^2}\right)}\left(1-\frac{y^2}{\sigma^2}\right)dy +\int_{-\mu}^0\exp{\left(-\frac{y^2}{2\sigma^2}\right)}\left(1-\frac{y^2}{\sigma^2}\right)dy = 0.
    \end{split}
    \end{equation}
    
    in the interval $y<-\mu (< -\sigma)$, the integrand is negative, thus
    \begin{equation}
    \int_{-\infty}^{-\mu}\exp{\left(-\frac{y^2}{2\sigma^2}\right)}\left(1-\frac{y^2}{\sigma^2}\right)dy < 0.
    \label{eqA3}
    \end{equation}
    
    \Cref{eqA3} can be written as
    \begin{equation}
    \begin{split}
        &\int_{-\infty}^{-\mu}\exp{\left(-\frac{y^2}{2\sigma^2}\right)}\left(1-\frac{y^2}{\sigma^2}\right)dy = -\int_{-\mu}^0\exp{\left(-\frac{y^2}{2\sigma^2}\right)}\left(1-\frac{y^2}{\sigma^2}\right)dy,
    \end{split}
    \end{equation}
    
    thus
    \begin{equation}
    \int_{-\mu}^0\exp{\left(-\frac{y^2}{2\sigma^2}\right)}\left(1-\frac{y^2}{\sigma^2}\right)dy > 0
    \end{equation}
    holds.
    
    Therefore, $g(\sigma) > 0 $ regardless of $\sigma$ if $\mu>0$, then $\operatorname{Erf}(0;\mu,\sigma)$ is monotonically increasing with respect to $\sigma^2$.
    \end{proof}
    %
    %
    \Cref{lemma:varub} suggests that, the maximum $\sigma^2$ gives the maximum value of $\operatorname{Erf}(0;\mu,\sigma)$. \Cref{lemma:varmin} suggests that the maximum value of the $\sigma^2$ is given by \cref{eq:var2}, thus the maximum value of the $\operatorname{Erf}(0;\mu,\sigma)$ can be calculated, when fixing $x_*$.
    
    When $x_*$ satisfies the condition $\theta_{r1} = r$, then $\theta_1, \theta_s, \theta_{r1}$ are constants. Thus $\operatorname{Erf}(0;\mu,\sigma)$ can be regarded as the function of $\theta_{r2}$.
    
    \begin{lemma}
    $\operatorname{Erf}(0;\mu,\sigma)$ is monotonically increasing with respect to $\theta_{r2}$.
    \label{lemma:monor2}
    \end{lemma} 
    
    \begin{proof}
    Set $\Phi(x)$ as the cumulative distribution function of $\mathcal{N}\left(x;0,1 \right)$.
    \begin{equation}
    \Phi(x) = \frac{1}{\sqrt{2\pi}}\int_{-\infty}^x\exp{\left(-\frac{u^2}{2}\right)}du
    \end{equation}
    
    Then $\operatorname{Erf}(0;\mu,\sigma)$ can be written as
    
    \begin{equation}
    \operatorname{Erf}(0;\mu,\sigma) = \Phi\left(-\frac{\mu}{\sigma}\right)
    \end{equation}
    
    $\Phi(x)$ is the cumulative distribution function of the Gaussian distribution, thus it is monotonically increasing with respect to $x$. Now 
    \begin{equation}
    \mu = \frac{\theta_{r1}-\theta_{r2}}{\theta_{1}-\theta_{s}}
    \end{equation}
    
    From the condition, $\theta_{r1} > \theta_{r2}$, therefore $\mu$ is positive.
    $\sigma$ is positive by definition, thus $-\frac{\mu}{\sigma}$ is negative.
    
    Set
    \begin{equation}
    g(\theta_{r2}) = \frac{\mu^2}{\sigma^2} = \frac{(\theta_{r1}-\theta_{r2})^2}{(\theta_{1}-\theta_{s})^2}\cdot\frac{1}{\sigma^2}
    \end{equation}
    
    When $g(\theta_{r2})$ is decreasing with respect to $\theta_{r2}$, then $\frac{\mu}{\sigma}$ is decreasing with respect to $\theta_{r2}$ thus $\Phi\left(-\frac{\mu}{\sigma}\right)$ is increasing with respect to $\theta_{r2}$.
    
    Below we show that $g(\theta_{r2})$ is decreasing with respect to $\theta_{r2}$. 
    
    the differentiation of $g(\theta_{r2})$ can be decomposed as
    \begin{equation}
    \begin{split}
    &\frac{\partial}{\partial\theta_{r2}} g(\theta_{r2}) = -\frac{2(\theta_{r1}-\theta_{r2})\left\{\theta_1(\theta_1+\theta_s)-\theta_{r1}(\theta_{r1}+\theta_{r2})\right\}}{\left\{-\theta_1\theta_{r2}^2+2\theta_s\theta_{r1}\theta_{r2}+\theta_1(\theta_1^2-\theta_s^2-\theta_{r1}^2)\right\}^2}.
    \label{eq:targlemma3}
    \end{split}
    \end{equation}
    The inequality below holds for any kernel functions because the determinant of the \textit{gram matrix} is positive:
    \begin{equation}
    k(x_+,x_+)k(x_-,x_-) > k(x_+,x_-)^2
    \end{equation}
    Now $k(x_+,x_+) = k(x_-,x_-)$ by the condition, thus $k(x_+,x_+) > k(x_+,x_-) \iff \theta_1 > \theta_s$ holds.
    
    From the condition of the theorem, $\theta_{r1}-\theta_{r2}>0$ holds.
    
    Thus,
    \begin{equation}
    \theta_1(\theta_1+\theta_s) > \theta_s(\theta_s+\theta_s) = 2\theta_s^2
    \end{equation}
    and
    \begin{equation}
    \theta_{r1}(\theta_{r1}+\theta_{r2}) < \theta_{r1}(\theta_{r1}+\theta_{r1}) = 2\theta_{r1}^2
    \end{equation}
    
    From the condition of the theorem, 
    \begin{equation}
        \theta_{r1} = r < \theta_s
        \label{eq:upr1st}
    \end{equation}
    holds, therefore
    \begin{equation}
    \theta_1(\theta_1+\theta_s) > 2\theta_s^2 > 2\theta_{r1}^2 >  \theta_{r1}(\theta_{r1}+\theta_{r2})
    \end{equation}
    
    and
    \begin{equation}
    \theta_1(\theta_1+\theta_s)-\theta_{r1}(\theta_{r1}+\theta_{r2}) > 0
    \label{eq:upr2nd}
    \end{equation}
    holds.

    Plugging \cref{eq:upr1st} and \cref{eq:upr2nd} into \cref{eq:targlemma3}, the partial differentiation is negative, thus $\operatorname{Erf}(0;\mu,\sigma)$ is monotonically increasing with respect to $\theta_{r2}$.
    \end{proof}
    
    From the discussions above, $x_{*max} = \argmax_{x_*}[k(x_-,x_*)]$ gives the maximum value of $\operatorname{Erf}(0;\mu,\sigma)$.
    
    Moreover, $\operatorname{Erf}(0;\mu,\sigma) = \Phi\left(-\frac{\mu}{\sigma}\right)$ is monotonically decreasing with respect to $\mu$, thus for any $\epsilon>0, \operatorname{Erf}(0;\mu,\sigma) < \operatorname{Erf}(0;\mu-\epsilon,\sigma)$ holds. 
    
    Now we prove Theorem 1.
    From the discussions above, the probability is upper-bounded as
    \begin{equation}
        \Pr(\mathcal{C}(x_*)=-1) < \int^0_{-\infty}\mathcal{N}\left(x;\mu-\epsilon,\sigma^2 \right)dx.
    \label{eq:proof1}
    \end{equation}
    Plugging the inequality below \cite{shafahi_are_2019} into \cref{eq:proof1}, that is, applying
    \begin{equation}
        \forall \alpha<0;~\Phi(\alpha) \leq \frac{1}{2}\exp{\left(-\alpha^2/2\right)}
    \end{equation}
    
    to the RHS of \cref{eq:proof1}, the below inequality
    \begin{equation}
        \Pr(\mathcal{C}(x_*)=-1) < \phi(r|\mathcal{D}) = \frac{1}{2}\exp{\left(-\frac{\mu^2}{2\sigma^2}\right)}
    \end{equation}
    holds.
    
    $\phi(x_*|\mathcal{D})$ is a \textit{ maximum success probability (MSP)} function
    where
    \begin{equation}
    \begin{split}
    &\mu = \frac{\theta_{r1}-\theta_{r2}}{\theta_{1}-\theta_{s}}-\epsilon,\\
    &\sigma^2 = \theta_1 - \frac{(\theta_1({\theta_{r1}}^2+{\theta_{r2}}^2)-2\theta_s\theta_{r1}\theta_{r2})}{{\theta_1}^2 - {{\theta_s}^2}},\\
    &\theta_{1} = k(x_+,x_+),~
    \theta_{r1} = k(x_+,x_{*max}),\\
    &\theta_{r2} = k(x_-,x_{*max}),~
    \theta_s = k(x_+,x_-),\\
    &x_{*max} = \argmax_{x_*}[k(x_-,x_*)].
    \end{split}
    \end{equation}
    \end{toappendix}
\end{theoremproof}

\subsection{Remarks on \cref{thm:msp}}
\begin{itemize}
\item $x_+$ and $x_-$ indicate the closest points of the two classes from the training dataset in the classification task since the kernel function can be regarded as the function whose output is the similarity between the data points.
Practically, $x_*$ can be regarded as an AE whose original data point is $x_+$ and the adversarial perturbation is $r$. AE is formalized as a sample whose distance from the original data points is a small value (or more specifically, a small adversarial perturbation norm) $r$ \cite{carlini_towards_2017}. 

\item This theorem indicates that when an AE is crafted using an original input point in GP with adversarial perturbation $r$, the probability that the AE is classified as a different class has a non-trivial upper bound, and that bound is determined as the function of the kernel function used in the GP regression, the distance between the original data point and the nearest data point from the different class. The distance is measured using the kernel function.

\item When the AE is classified as a different class, it can be regarded as a successful attack. 
Thus, the theorem indicates that the probability of a successful attack using AE is upper-bounded.

\item We use GP regression for the classification task. The method using a regressor for classification instead of a classifier was also used in \cite{lee_deep_2018} and produced a good result.

\item Intuitively, the probability of successful adversarial examples should be upper-bounded by the distance from the nearest sample to the decision boundary. However, in general, the decision boundary cannot be easily calculated. Instead of considering the decision boundary, we proved that finding the closest points from different classes is sufficient to investigate the robustness of GP.
\end{itemize}

\subsection{Maximum Success Probability of AE within All Data in the Training Set}
In this section, we prove the maximum success probability of AE in all data in the training set. 

The problem formulation is the same as \cref{thm:msp}.
Let $S=\{s_{11},s_{12},\cdots\, s_{mn}\}$ where $s_{ij} = k(x_i, x_j), x_i \in \mathcal{D}_+$ and $x_j \in \mathcal{D}_-$. Set $s$ as the maximum of the set $S$, and put $ x_+ \in \mathcal{D}_+, x_- \in \mathcal{D}_-$ such that $k(x_+,x_-) = s$. Set $r \in \mathbb{R}$ such that for all $x \in \mathcal{D}_+$ and $x_* \in \mathbb{R}^D$ such that $k(x,x_*) = r$ and $k(x,x_*) > k(x_-,x_*)$ holds.

\begin{Theorem}
Given $x_+$ and $x_-$ as above, and given $x_{\circ}$ arbitrarily from $\mathcal{D}_+$, set $x_{+*}, x_{\circ*} \in \mathbb{R}^D$ such that $k(x_+,x_{+*}) = r$ and $ k(x_\circ,x_{\circ*})=r$.
Then, for any $r > k(x_-,x_{+*})$, $x_\circ \in \mathcal{D}_+$, and $x_{\circ*} \in \mathbb{R}^D$, $\Pr(\mathcal{C}(x_{\circ*})=-1)$ is upper-bounded by the upper bound of $\Pr(\mathcal{C}(x_{+*})=-1)$, when \textit{MSP} function $\phi(r|\mathcal{D})$ in \cref{thm:msp} increases monotonically with respect to $k(x_+,x_-)$. That is,
$    \Pr(\mathcal{C}(x_{\circ*})=-1) \leq \phi(r|\mathcal{D}) $ holds.

\label{thm:all}
\end{Theorem}
The schematic diagram is shown in \cref{fig:thrmspall}.
\definecolor{bb}{rgb}{0.63, 0.79, 0.95}
\definecolor{or}{rgb}{1, 0.8, 0.64}
\begin{figure}
    \centering
\begin{tikzpicture}[scale=0.4]
    
    \draw[dashed](2.6,7)--(4.6,-.5);
    \coordinate [label=above left:{O}] (origin) at (0,0);
    \coordinate [label={[blue]\large $\mathcal{D}_+$}] (bluelb) at (1.8,6.0);
    \coordinate [label={[orange]\large $\mathcal{D}_-$}] (orangelb) at (4.1,6.0);
    
    \fill [bb, rotate around={20:(1,4)}] (0.5,3) circle (2 and 3);
    \fill [or,rotate around={20:(1,4)}] (6,2) circle (2 and 3);

    \draw ([shift={(4.17,1.71)}]80:1.81) arc[radius=1.81, start angle=80, end angle= 135];
    
    \draw[very thick] (2.9,3)--(4.44,3.5);
    \fill[blue] (2.9,3) circle (0.12);
    \fill[orange] (4.44,3.5) circle (0.12);

    \coordinate [label=below:{$x_+$}] (bp) at (2.9,2.8);
    \coordinate [label=below:{$x_-$}] (op) at (4.44,3.4);
    
    \coordinate [label=above:{$s=\max{(S)}$}] (rlabel) at (4.7,3.65);

    \draw[thick, ->] (-2, 0)--(10, 0) node[right]{};
    \draw[thick, ->] (0,-0.5)--(0,7) node[above] {};
    
\end{tikzpicture}
    \caption{\small{Illustration of the conditions in Theorem \ref{thm:all} assuming the input space as $\mathbb{R}^2$. Blue and orange circles suggest the distributions of the input data points of $\mathcal{D}_+$ and $\mathcal{D}_-$ respectively. $x_+$ and $x_-$ are the nearest input data points from $\mathcal{D}_+$ and $\mathcal{D}_-$ respectively. Note that $s$ is the value of the kernel function whose inputs are the nearest points from the dataset, that is, $s=k(x_+,x_-)$.}}
    \label{fig:thrmspall}
\end{figure}
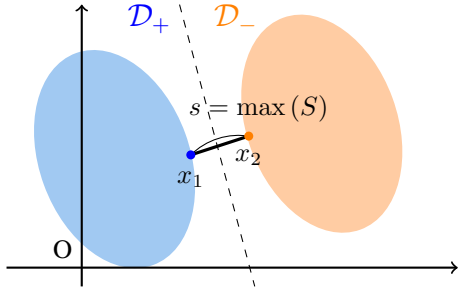

    \begin{proof}
    Where $r$ is fixed, $\phi(r|\mathcal{D})$ can be regarded as a function of $k(x_+,x_-)$ provided that the kernel function is translation invariant. Therefore, when $\phi(r|\mathcal{D})$ in \cref{thm:msp} increases monotonically with respect to $k(x_+,x_-)$, the greater value $\phi(r|\mathcal{D})$ is given by bigger $s$. From the condition, the value of the kernel function whose inputs are $x_+$ and $x_-$ is the greatest among the values of the kernel function with the pairwise chosen points from $\mathcal{D}_+$ and $\mathcal{D}_-$. Thus, the upper bound of $\Pr(\mathcal{C}(x_{+*})=-1)$ is greater than the upper bound with any $x_\circ \in \mathcal{D}_+$ and the correspondent point in $\mathcal{D}_-$ used. \qed
 \end{proof}

\subsection{Remarks on \cref{thm:all}}
\label{subs:remthm2}
\begin{itemize}
    \item \cref{thm:all} indicates that the AE crafted with the original data point more successfully when the original data point is closer to the image from another class if the kernel function meets the condition that $\phi(r|\mathcal{D})$ is monotonically increasing as a function of $k(x_+,x_-)$.
    \item If the kernel function has the condition $\phi(r|\mathcal{D})$ increases monotonically as a function of $k(x_+,x_-)$, it suggests that one of the points of the closest pair is the original data point from which the AE is crafted most successfully in GP classification with that kernel function.
\end{itemize}

\section{Experimental Results}

\label{sec:3}
For justification of \cref{assumption:1} and confirmation of whether the theorem holds for GP classification with a standard dataset, we conducted the following experiment.

\subsection{Devices}
The experiments were conducted on Ubuntu 20.04.3 LSB machine, Intel(R) Xeon(R) Gold 6140 CPU @ 2.30GHz (72 cores, 144 threads).
We used Python 3.8.10, numpy 1.23.5, and scipy 1.10.1 for the calculation.

\subsection{Materials}
We used ImageNet as a dataset. ImageNet is used under the license written on this page \url{https://www.image-net.org/download.php}.
The ImageNet data used are downloaded from \url{https://www.kaggle.com/datasets/liusha249/imagenet10}.

\subsection{Procedure}
We used the image of 10 classes from ImageNet, labeled as $0,1,\dots, 9$ respectively.
We chose samples from ImageNet, whose labels are either $a$ or $b$ ($a$ and $b$ are taken from the labels $\{0,1,\dots,9\}$ pairwisely, which results in 90 combinations). We used 500 samples for each label. 
The data used in the experiment were those with 3 colors, and randomly cropped to $224 \times 224$ pixels to adjust the image size with different aspect ratios. Therefore, the input dimension of the data was 150528.
We trained the GP regressor \cite{rasmussen_gaussian_2006} using these samples as training data with the Gaussian kernel. 
The Gaussian kernel used in the experiment is as follows:
\begin{equation}
    k(x,x') = \theta_1 \exp{\left(-\frac{||x-x'||^2}{\theta_2}\right)}
    \label{eq:gkernel}
\end{equation}
In the experiment, the adversarial perturbation was fixed to 10, and the parameters of the kernel function changed. That is, the parameters $\theta_1 = 0.1,0.5,1$ and $\theta_2 = 10,50$ were used. The randomly cropped area of each image is the same to compare the results across the conditions.

We decided on the adversarial perturbation size considering the distance between a point and the closest point of the other class (see \Cref{fig:hist}). Most points are not classified as the other class even when the adversarial perturbation of size 10 is added to the points.

We gave the objective variable -1 to the points with label $a$, and 1 to the points with label $b$.

We crafted AE as the point with the following conditions:
\begin{itemize}
    \item Each AE is crafted by adding perturbation to the original sample of label $b$, resulting in crafting 500 AEs.
    \item The perturbation is given by moving the original sample by a certain distance ($l_2$ norm) towards the nearest point whose label is $a$.
\end{itemize}
A sample of adversarial examples crafted in this experiment is shown in \cref{fig:advexp}.
\begin{figure}
    \centering
    \subfloat[][Original] {\includegraphics[width=0.4\linewidth]{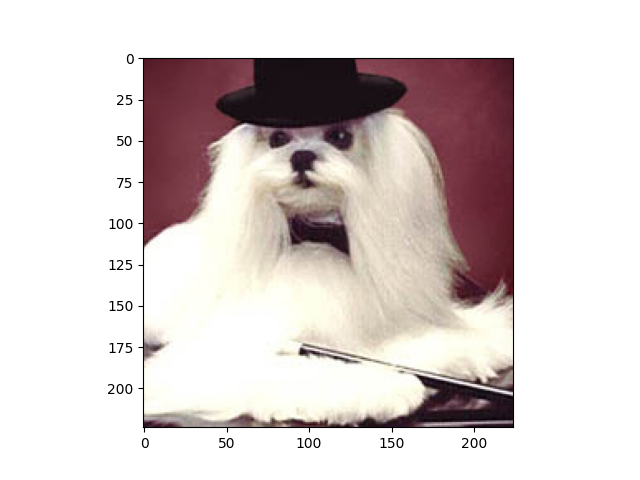}}
    \subfloat[][AE] {\includegraphics[width=0.4\linewidth]{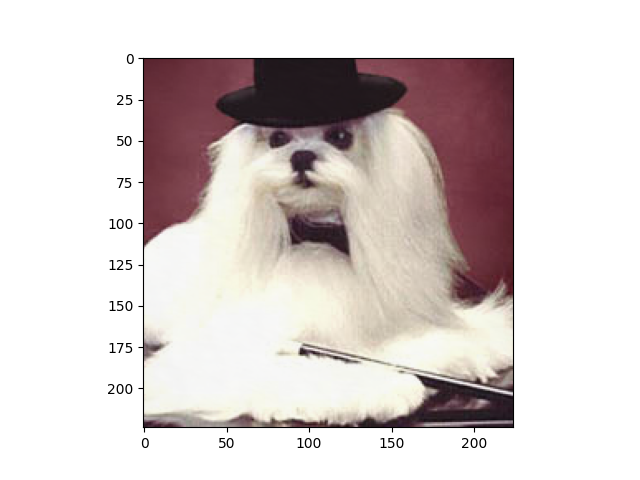}}
    \caption{\small{A sample of adversarial examples of norm=$10$ crafted in the current experiment. The adversarial perturbation is directed to the nearest point in the other class.}}
    \label{fig:advexp}
\end{figure}

Next, we calculated the predicted mean $\mu_\mathcal{R}$ and variance $\sigma^2_\mathcal{R}$ using the GP regressor, and the probability that the regression result is negative (i.e. the classification result is labeled $b$) given by
\begin{equation}
    \int_\infty^0\mathcal{N}\left(x;\mu_\mathcal{R},\sigma^2_\mathcal{R}\right)dx.
    \label{eq:emp}
\end{equation}

This is the empirical probability of the classification of an AE as the other class than the original class in the GP classifier.

Finally, we calculated the theoretical probability of classification of an AE as the other class than the original class, using $\Phi\left(-\mu/\sigma\right)$ in \cref{thm:msp} with $\epsilon = 0$.

These theoretical and empirical probabilities were calculated for all the 500 points that had the label $b$ in each condition. In each condition, the mean and maximum of the theoretical and empirical probabilities were calculated and their mean across the conditions was calculated.

\subsection{Result}
\begin{figure}[htbp]
    \centering
    \subfloat[][a=0 and b=7, \\norm=10, \\$\theta_1 = 0.1$ and $\theta_2 = 10$] {\includegraphics[width=0.3\textwidth]{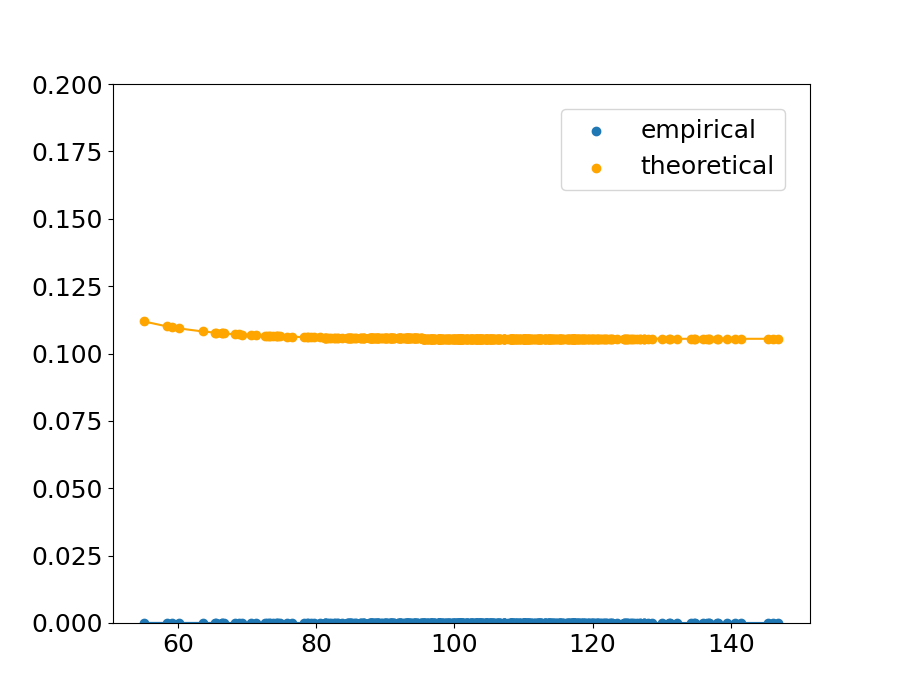}}
    \subfloat[][a=0 and b=7, \\norm=10, \\$\theta_1 = 0.5$ and $\theta_2 = 10$] {\includegraphics[width=0.3\textwidth]{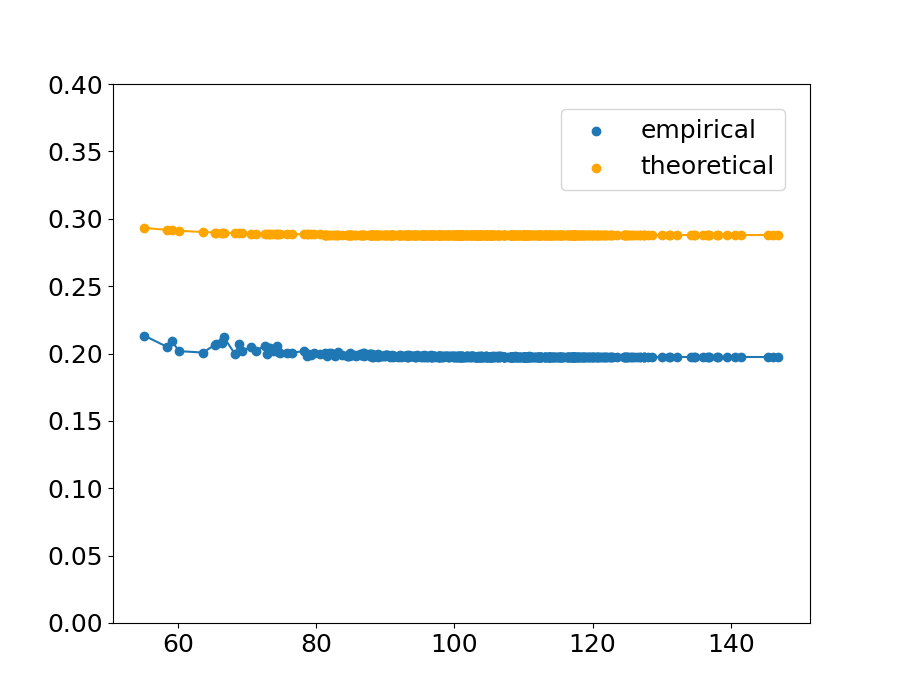}}
    \subfloat[][a=0 and b=7, \\norm=10, \\$\theta_1 = 1$ and $\theta_2 = 10$] {\includegraphics[width=0.3\textwidth]{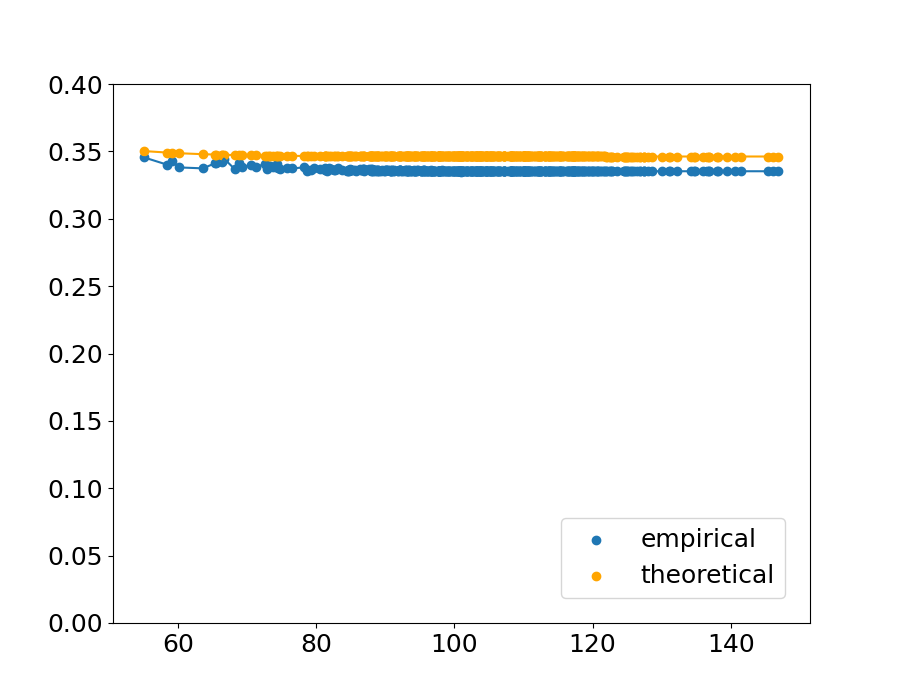}}
    
    \subfloat[][a=0 and b=7, \\norm=10, \\$\theta_1 = 0.1$ and $\theta_2 = 50$] {\includegraphics[width=0.3\textwidth]{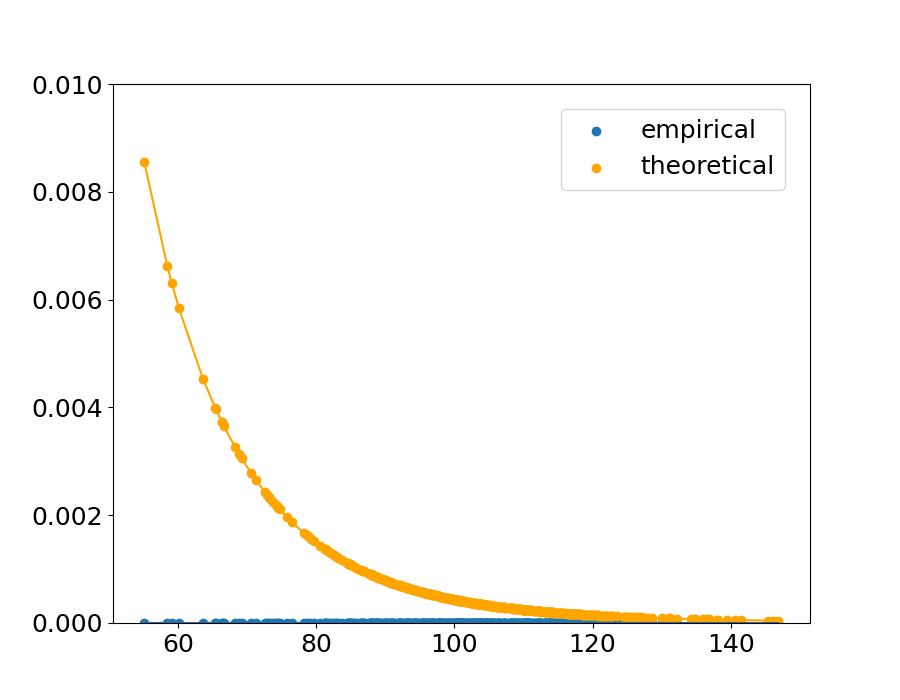}}
    \subfloat[][a=0 and b=7, \\norm=10, \\$\theta_1 = 0.5$ and $\theta_2 = 50$] {\includegraphics[width=0.3\textwidth]{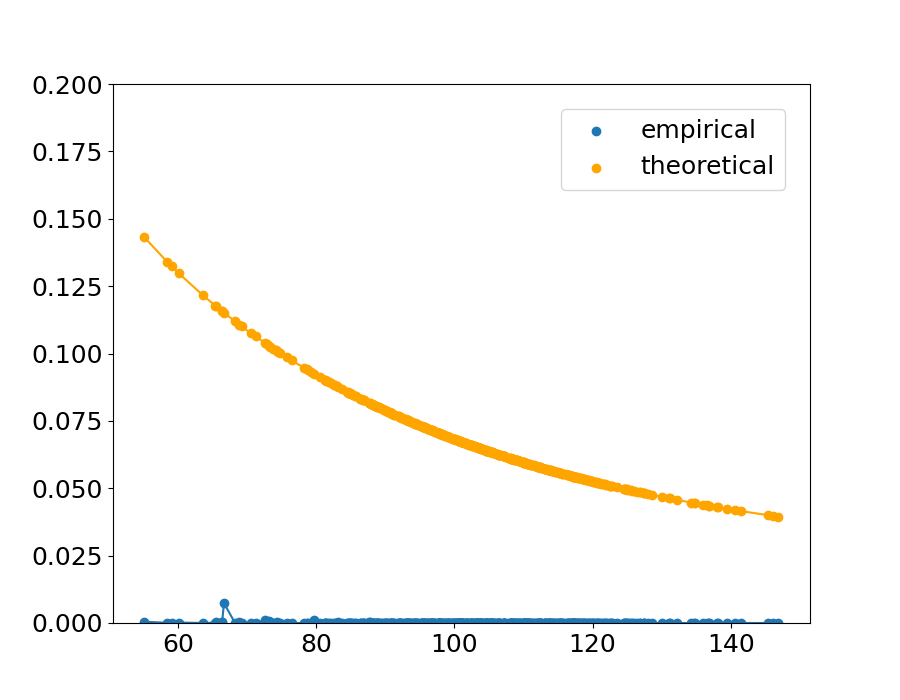}}
    \subfloat[][a=0 and b=7, \\norm=10, \\$\theta_1 = 1$ and $\theta_2 = 50$] {\includegraphics[width=0.3\textwidth]{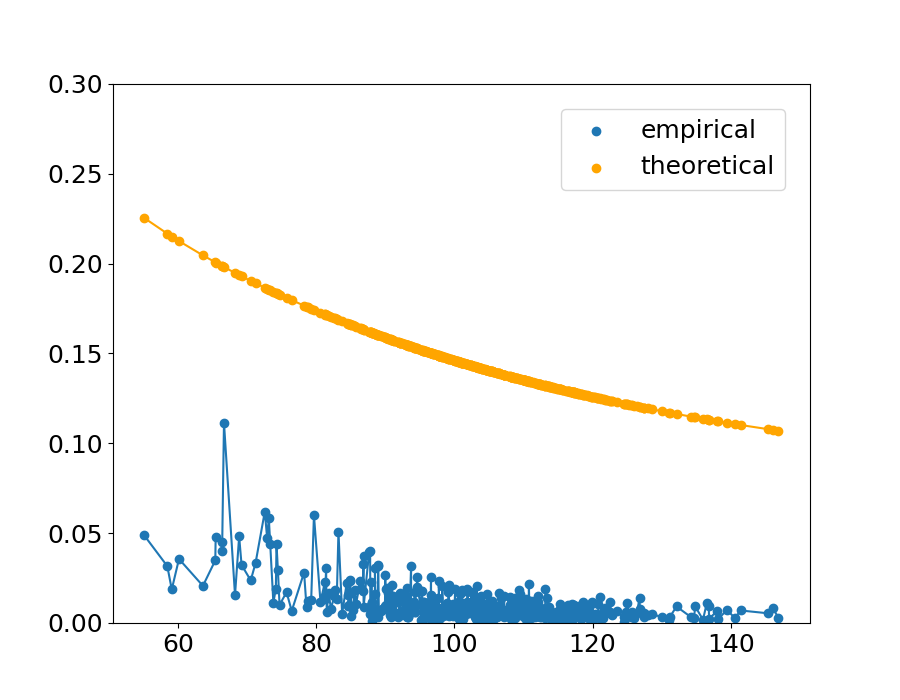}}    

    \caption{\small{The samples from the result of the experiment. The horizontal axis shows the distance between a certain point and the closest point whose label is different from that point. The vertical axis shows the empirical and theoretical probability that the point is classified as a different label from the closest point. Note that the theoretical upper bound changes according to the kernel parameter $\theta_1$ and $\theta_2$.}}
    \label{fig:expimgnet}
\end{figure}

\begin{figure}[htbp]
    \centering
    \includegraphics[width=0.4\linewidth]{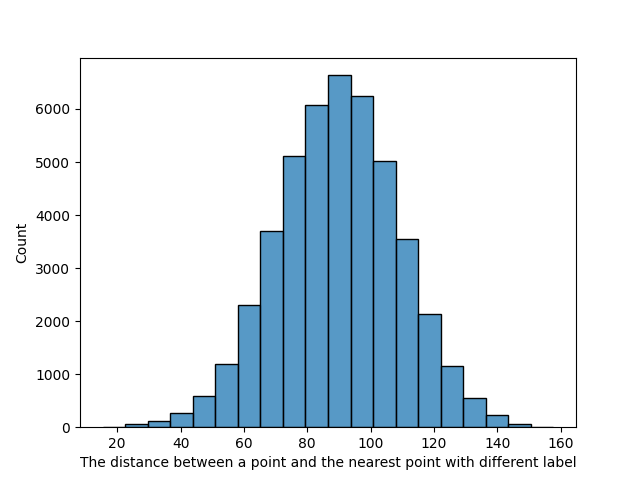}
    \caption{\small{The histogram of the distance between a point and the nearest point with a different label. The data pool is merged across the pairwise condition, so the number of the data is $45000 (= 500 * 90)$.}}
    \label{fig:hist}
\end{figure}

\cref{fig:expimgnet} shows the sample of the theoretical upper bound and the empirical value in the condition $a=0 ~\text{and}~ b=7$. The horizontal axis indicates the distance between a point and the nearest point from the other class, and the vertical axis indicates the theoretical upper bound calculated with \cref{thm:msp} and the empirical value calculated with \cref{eq:emp}. The result is shown in \cref{tab:resultexp2}, which indicates the proportion of the points that follow the theorem, and the maximum theoretical value within the $a$ and $b$ combination (the values shown in \cref{tab:resultexp2} are the means across the 90 combinations $a$ and $b$).

The analysis of the distance between the points and their nearest points that have different labels is shown in \cref{tab:distexp2}. The mean of the distance is $89.72$, and if the points are restricted to those that do not follow the theorem (that is, their theoretical values are not larger than the empirical values), the distance is smaller.
The overall histogram of the distance between a point and the nearest point with a different label is shown in \Cref{fig:hist}.

An extreme example is shown in (c) of \Cref{fig:expimgnet}. In this condition of kernel parameters, the empirical probability is high and in that condition the theoretical upper bound is tight.

The upper bound does not always hold because in this experiment, the upper bound is calculated with \cref{thm:msp} with $\epsilon=0$, while in the real data, $\epsilon$ could be positive. However, the fact that the theorem holds with high probability even if $\epsilon$ is $0$ suggests that we can empirically say that \cref{assumption:1} is satisfied in practice.

Theoretically, when the bandwidth of the Gaussian distribution in the kernel function is small, the empirical decision boundary is not affected by the distant points and is determined mainly by the closest points. In that condition, the theoretical upper bound (calculated by two nearest points) is closer to the empirical probability.

\begin{table*}[htbp]
    \centering
    \footnotesize
    \caption{\small{The proportion of points where the theoretical value is larger than the empirical value and the mean ($\pm$ standard deviation) of the theoretical upper bound of the closest pair in each condition. The change of kernel parameters $\theta_1, \theta_2$ cause the change of the theoretical upper bound.}}
    \begin{tabular}{c|cccccc}
         $\theta_1$ & 0.1 & 0.1 &0.5 & 0.5& 1 & 1 \\
         $\theta_2$ &  10 & 50 & 10 & 50 & 10 & 50\\
         \hline \hline
         \begin{tabular}{c}
         The proportion of points \\that follow the theorem
         \end{tabular}&0.9996 &0.9999&0.9991 & 0.9998 & 0.9898& 0.9997\\
         \hline
         The mean of the & 0.2095 &0.1216 & 0.3512 &0.2616 & 0.3928 & 0.3216\\
         max theoretical value &$\pm$ 0.1538 & $\pm$ 0.1927 & $\pm$ 0.0809 & $\pm$0.1291 & $\pm$0.0585 & $\pm$0.0981 \\
    \end{tabular}
    \label{tab:resultexp2}
\end{table*}

\begin{table*}
    \centering
    \footnotesize
    \caption{\small{The first row suggests the mean ($\pm$ standard deviation) of the distance between the points which doesn't follow the theorem and their nearest points from the other class. The second row suggests the mean of the distance between the points and their nearest points and their nearest points from the other class.}}
    \begin{tabular}{c|cccccc}
          $\theta_1$ & 0.1 & 0.1 & 0.5 & 0.5 & 1& 1 \\
          $\theta_2$ &  10& 50 & 10 & 50 & 10 & 50\\
         \hline \hline

        The mean distance between the points
          & 23.59 & 18.55 & 31.13 & 24.58 & 50.84& 30.13 \\
           which don't follow the theorem &$\pm$ 5.342 & $\pm$4.031 & $\pm$ 8.524 & $\pm$13.25 & $\pm$ 10.99 & $\pm$14.98 \\
          \hline
          \begin{tabular}{c}
          The mean of the distance \\
          of all points
          \end{tabular} & \multicolumn{6}{|c}{89.72} \\
    \end{tabular}
    \label{tab:distexp2}
\end{table*}

\section{Discussion}

\subsection{The Interpretation of the Theorems}
Our theorems show the fundamental limitation of the AE in GP; the success probability of AE crafted by any method cannot exceed the value of MSP function $\phi$ in \cref{thm:msp}, and the vulnerability of a training dataset to the AE is determined by the closest pair of the data points whose labels are different. 

The theorems can be applied to other inference models because GP includes some inference models. In particular, the theoretical bound can be computed for neural networks. Since neural networks with an infinite number of units in the hidden layer and Bayesian neural networks are regarded GP \cite{gal_dropout_2016,williams_computing_1996}, our results apply to neural networks, resulting in the fundamental limitation of AE in neural networks, as in previous research \cite{blaas_adversarial_2020,fawzi_adversarial_2018,shafahi_are_2019}. Compared with previous research, our result adds the viewpoint of the distance between the samples.

\subsection{The Choice of Kernel Functions}
Since the predictive mean can be written as the sum of the terms of the kernel functions whose inputs are each training point and test point (Representer Theorem in GP \cite{goos_generalized_2001}), changing the parameter of the kernel functions changes the shape of the decision boundary.

The result of the experiments in \cref{sec:3} suggests that the upper bound of the probability changes according to the choice of the kernel functions. This suggests that changing the activation function of a neural network can improve the robustness of that neural network.
The change according to the kernel parameter of the Gaussian kernel (\cref{eq:gkernel}) shown in the experiment in \cref{sec:3} can be interpreted as below:
\begin{itemize}
    \item When $\theta_1$ changes, the value of the Gaussian kernel is multiplied by $\theta_1$. With regard to $\mu$ and $\sigma$ in \cref{thm:msp}, when $\theta_1$ becomes larger, $\mu$ will remain unchanged and $\sigma$ will be larger. Considering that the \textit{MSP} function (\cref{eq:msp}) is based on the CDF of Gaussian distribution and the CDF of Gaussian distribution increases monotonically with respect to $\sigma$ when $\mu>0$, the value of \cref{eq:msp} is greater when $\theta_1$ is greater.
    \item When $\theta_2$ changes, the Gaussian distribution bandwidth in the Gaussian kernel will change. Considering the form of the exponential function, if $\theta_2$ is larger, the absolute value of the derivative of the Gaussian kernel with respect to $\lVert x-x' \rVert^2$ at the same point is smaller. This suggests that if $\theta_2$ is greater, the decay of the value of the Gaussian kernel with respect to $\lVert x-x' \rVert^2$ is slower, and the effect on the value of the Gaussian kernel by the distance of two points would be greater in the interval focused in the experiment.
\end{itemize}

Since $\phi(r|\mathcal{D})$ is not easily differentiable with respect to $k(x_1,x_2)$, the discrete calculation is required to confirm whether the kernel function satisfied the constraint in \cref{subs:remthm2} and to determine the kernel function to improve the robustness. 
However, it is reasonable that if $k(x_1,x_2)$ is small (that is, $x_1$ and $x_2$ is too far away), $x_*$, the point near $x_1$, cannot be classified as the same class as $x_2$. 

Previous research showed that the activation functions in neural networks are equivalent to the kernel functions in GP\cite{lee_deep_2018,williams_computing_1996}. Therefore, the theoretical result of this paper suggests the theoretical basis for the enhancement method that changes the activation function in neural networks. Previous studies suggest that changing the activation function in neural networks can enhance the robustness against AEs \cite{goodfellow_explaining_2015,gowal_uncovering_2021,taghanaki_kernelized_2019}. Our result shows that the change of the kernel function induces the change of the theoretical upper bound of the successful AEs, and considering that the activation function in neural networks is the counterpart of the kernel function in GP, our result suggests that the robustness enhancement method by changing the activation functions in neural networks has the same basis as our result.

\subsection{Open Problems}
This paper has the following open problems: Further investigation should solve these problems.

\begin{itemize}
    \item \cref{thm:msp} is proved under the assumption that the effect of input points other than $x_1, x_2$ on the predictive mean is less than $\epsilon$. When using the Gaussian kernel, this condition is reasonable because the value of the Gaussian kernel decreases rapidly according to the distance of the input points, and it is confirmed in the experiment that even if $\epsilon$ is set to 0, over 98\% of the points follow the theorem. The theorem under the condition of using another kernel function should be investigated further. 
    \item An experiment with neural networks should be conducted to confirm that one can design a more robust activation function using \cref{thm:msp}.
\end{itemize}

\subsection{Limitations}
The research has the following limitations.
\begin{itemize}
    \item The problem formulation of the research is binary classification, not multi-class classification. 
    \item Finding the nearest point can be time-consuming if the number of input points increases. An efficient algorithm, such as a divide-and-conquer algorithm, should be investigated for the search for the nearest point.
    \item $\epsilon$ in \cref{assumption:1} depends on the distribution of the training data, the evaluation of the size of $\epsilon$ is yet to be investigated (but as we wrote above, the fact that the theorem holds with high probability even if is 0 suggests that we can safely use the upper bound without \cref{assumption:1} in practical use).
\end{itemize}

\subsection{Conclusion}
In this paper, we showed that in the GP classification, the probability of a successful attack of AE has an upper bound. The experiment showed that the parameter of kernel functions can change the theoretical upper bound of the probability of a successful attack of AE, suggesting that the choice of kernel functions affects the robustness against AE in the GP classification. 


\small{ \subsubsection{Disclosure of Interests.} The authors have no competing interests to declare that are relevant to the content of this article.}


%
%
%
\bibliographystyle{splncs04}
\bibliography{main}
\end{document}